\documentclass[letterpaper]{article} 
\usepackage[submission]{aaai24}  
\usepackage{times}  
\usepackage{helvet}  
\usepackage{courier}  
\usepackage[hyphens]{url}  
\usepackage{graphicx} 
\usepackage{natbib}  
\usepackage{caption} 
\usepackage{algorithm}
\usepackage{algorithmic}

\usepackage{amsmath}
\usepackage{amsthm}
\usepackage{amssymb}
\newtheorem{proposal}{Proposal}

\usepackage{newfloat}
\usepackage{listings}
\DeclareCaptionStyle{ruled}{labelfont=normalfont,labelsep=colon,strut=off} 
\floatstyle{ruled}
\newfloat{listing}{tb}{lst}{}
\floatname{listing}{Listing}
\title{Learning to Recover for\\Safe Reinforcement Learning}
\author{
    Written by AAAI Press Staff\textsuperscript{\rm 1}\thanks{With help from the AAAI Publications Committee.}\\
    AAAI Style Contributions by Pater Patel Schneider,
    Sunil Issar,\\
    J. Scott Penberthy,
    George Ferguson,
    Hans Guesgen,
    Francisco Cruz\equalcontrib,
    Marc Pujol-Gonzalez\equalcontrib
}
\affiliations{
    \textsuperscript{\rm 1}Association for the Advancement of Artificial Intelligence\\

    1900 Embarcadero Road, Suite 101\\
    Palo Alto, California 94303-3310 USA\\
    proceedings-questions@aaai.org
}

\usepackage{bibentry}

\begin{document}

\maketitle

\begin{abstract}
Safety controllers is widely used to achieve safe reinforcement learning. Most methods that apply a safety controller are using handcrafted safety constraints to construct the safety controller. However, when the environment dynamics are sophisticated, handcrafted safety constraints become unavailable. Therefore, it worth to research on constructing safety controllers by learning algorithms. We propose a three-stage architecture for safe reinforcement learning, namely TU-Recovery Architecture. A safety critic and a recovery policy is learned before task training. They form a safety controller to ensure safety in task training. Then a phenomenon induced by disagreement between task policy and recovery policy, called adversarial phenomenon, which reduces learning efficiency and model performance, is described. Auxiliary reward is proposed to mitigate adversarial phenomenon, while help the task policy to learn to recover from high-risk states. A series of experiments are conducted in a robot navigation environment. Experiments demonstrate that TU-Recovery outperforms unconstrained counterpart in both reward gaining and constraint violations during task training, and auxiliary reward further improve TU-Recovery in reward-to-cost ratio by significantly reduce constraint violations.
\end{abstract}

\section{Introduction}

Reinforcement learning has been widely believed as key path to general artificial intelligence. 
In the past decade, deep reinforcement learning has made great achievements in many fields.
However, most of the successful reinforcement learning applications have been confined to virtual domains, 
such as video games, board games, and virtual robot controls, 
while the application of reinforcement learning in real-world is still a difficult problem that worthy of attention.
Safety is one of the main obstacles to the real-world application of reinforcement learning.
According to traditional reinforcement learning framework, agents are encouraged to explore environments, 
which is crucial to performance improvements.
What prevent reinforcement learning algorithms from being safe is exactly the exploration procedures in these algorithms, 
because any freely exploring action could put an agent into catastrophe if the environment contains danger areas.
Therefore, the exploratory nature of reinforcement learning contradicts safety requirements in lots of real-world applications, 
and how to balance exploration and safety -- that is, make RL agents gain performance improvements continuously while meet environment constraints -- 
is a tough but critical problem to safe RL.
Recently in some high-safety required area, such as autonomous driving, AI medical treatment, and robot navigation, 
researchers are showing growing interest in safe reinforcement learning.

Typically, safe reinforcement learning is achieved by adding constraints to the original optimization problem.
These constraints are usually characterized by handcrafted features. However, handcrafted features have their own disadvantages. 
The first disadvantage is that constraints constructed with handcrafted features are difficult, or even impossible to establish, 
when the environment dynamics operate in a complex way. The second disadvantage is that, 
it requires human-level prior knowledge to make handcrafted features. In some cases, no prior knowledge exists, 
which means building handcrafted features is infeasible. 
So in these scenarios, building safety constraints through learning algorithms is necessary. 
It is worth noting that, safety constraints need not to be explicit formulas. 
One can use a safe controller to supervise agent's exploration procedure, so that safety constraints are implicitly expressed by the safe controller.

This work mainly focuses on dealing with complex safety constraints. We first train a safety critic, and use this critic 
to train a task-unaware safe controller. Then task training is performed, under the supervision of safe controller.
We find that task actions, which proposed by task agent in order to maximize expected reward, and safe actions, which proposed by safe controller in order to ensure safety, 
are sometimes against each other. We call this phenomenon as "adversarial phenomenon". 
We come up with an idea that use auxiliary rewards to handle the adverse effects bring by this phenomenon. 
Contributions of this paper are listed below: 

\begin{itemize}
    \item We propose a new architecture -- TU-Recovery Architecture -- for safe reinforcement learning, where safety constraints and safety controllers are learned in advance of task learning procedure, and are used to guide the exploration of task agents.
    \item We propose to utilize auxiliary rewards to mitigate the adversarial phenomenon during task training. These auxiliary rewards also help task agents learn to take safe actions when situations get bad.
    \item We test the proposed methods in a robot navigation environment, and show that TU-Recovery outperforms unconstrained RL algorithms. Furthermore, we show that auxiliary reward improve TU-Recovery by significantly reduce constraint violations.
\end{itemize}


\section{Preliminaries}

Markov Decision Process (MDP) is an important model in the field of RL. A MDP can be represented by a six-tuple, $(\mathcal{S},\mathcal{A},P,r,\rho_0,\gamma)$, where $\mathcal{S}$ is the space of all possible states, $\mathcal{A}$ is the space of all possible actions, $P:\mathcal{S} \times \mathcal{A} \times \mathcal{S} \rightarrow [0,1]$ is the transition probability function that determines the probability of transitioning to another state after taking an action from one state, $r:\mathcal{S} \rightarrow \mathbb{R}$ is the reward function that indicates the reward of being in a state, $\rho_0: \mathcal{S} \rightarrow [0,1]$ is the distribution of initial states, and $\gamma$ is a discount factor. A stationary policy in this MDP can be represented by $\pi : \mathcal{S} \times \mathcal{A} \rightarrow [0,1]$, where $\pi(a\vert s)$ is the probability of take action $a$ under state $s$.

In traditional RL framework, an agent needs to learn a policy that maximize the expected discounted reward. That is, an agent aims to solving the following optimization problem:

\begin{align*}
    \mathop{argmax}\limits_{\pi} \quad \mathop{\mathbb{E}}\limits_{\tau \sim \pi} \left[ \sum_{t=1}^{\infty} \gamma^t r_{t} \right],
\end{align*}
where $\tau=(s_{0},a_{0},s_{1},a_{1},...)$ represent one trajectory generated by the interaction between the agent and the environment, $r_{t}$ is a simplified denotation of $r(s_{t})$, and by $\tau \sim \pi$ we mean that the distribution of trajectories follows the policy $\pi$.

Value functions are important notions in many RL algorithms. One value function, known as state value function, can be denoted as $V_{\pi}(s)=\mathop{\mathbb{E}}\limits_{\tau \sim \pi}\left[ \sum_{t=0}^{\infty}\gamma^{t} r_{t} \vert s_{0}=s \right]$. Another value function, known as state-action value function, is denoted as $Q_{\pi}(s,a)=\mathop{\mathbb{E}}\limits_{\tau \sim \pi}\left[ \gamma^{t} r_{t} \vert s_{0}=s,a_{0}=a \right]$.

Different from traditional RL, safe RL is usually modeled as Constrained Markov Decision Process (CMDP), which is an extension of MDP. A CMDP can be represented by a seven-tuple of the form $(\mathcal{S},\mathcal{A},P,r,c,\rho_0,\gamma)$, where the definitions of $\mathcal{S}$, $\mathcal{A}$, $P$, $r$, $\rho_0$ and $\gamma$ are the same as in MDP, and $c:\mathcal{S} \rightarrow \mathbb{R}$ is the cost function, which indicates the cost of being in a state.

In the context of safe RL, an agent considers both maximizing expected discounted reward and satisfying safety constraints. Typically, a safety constraint can be characterized by the expected discounted cost, $\mathop{\mathbb{E}}\limits_{\tau \sim \pi} \left[ \sum_{t=1}^{\infty} \gamma^t c_{t} \right]$. Therefore, the following constrained optimization problem is what the agent aims to solve:

\begin{align*}
    \mathop{argmax}\limits_{\pi} \quad & \mathop{\mathbb{E}}\limits_{\tau \sim \pi} \left[ \sum_{t=1}^{\infty} \gamma^t r_{t} \right], \\
    s.t. \quad & \mathop{\mathbb{E}}\limits_{\tau \sim \pi} \left[ \sum_{t=1}^{\infty} \gamma^t c_{t} \right] \leq d,
\end{align*}
where $d$ is a hyperparameter. In CMDP, cost value functions can be defined in the same form as value functions. In other words, we define state cost value function $V_{\pi}^{c}(s)$ and state-action cost value function $Q_{\pi}^{c}(s,a)$ in analogy to $V_{\pi}(s)$ and $Q_{\pi}(s,a)$ respectively, using cost, instead of reward, in the definition.

We assume that the cost function is an indicator function. The state space is split in to two disjoint subsets, safe state set and unsafe state set, denoted as $\mathcal{S}_{safe}$ and $\mathcal{S}_{unsafe}$ respectively. The cost function is written as:

\begin{equation}
    \label{equ:CostSignal}
    c(s) = \mathbf{1}_{\mathcal{S}_{unsafe}}(s)
\end{equation}

One way to safe RL is utilizing a safety controller. Figure \ref{fig:Safe_RL} shows a general framework for safety controller guided safe RL. According to this framework, the agent does not directly interact with the environment; Actions proposed by the agent is input into a safety controller, which always outputs safe actions. Suppose that the agent proposes an action, $a_{task}$, which is only in consider of maximizing the expected discounted reward. The safety controller can be thought as a function that map the entire action space to the safe action space, denoted as $\Phi$. So the actual action to be performed is $a_{safe}=\Phi(a_{task})$.

The design of safety controller is a crucial part in this framework. A good design of the safety controller can greatly improve the performance in gaining reward and satisfying safety constraints. 

\begin{figure}[t]
\centering
\includegraphics[width=0.9\columnwidth]{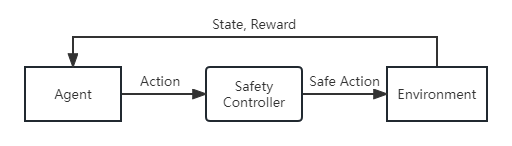}
\caption{A Typical Safe Reinforcement Learning Framework with Safety Controller}
\label{fig:Safe_RL}
\end{figure}

\section{Related Work}

\begin{figure*}[t]
\centering
\includegraphics[width=0.9\textwidth]{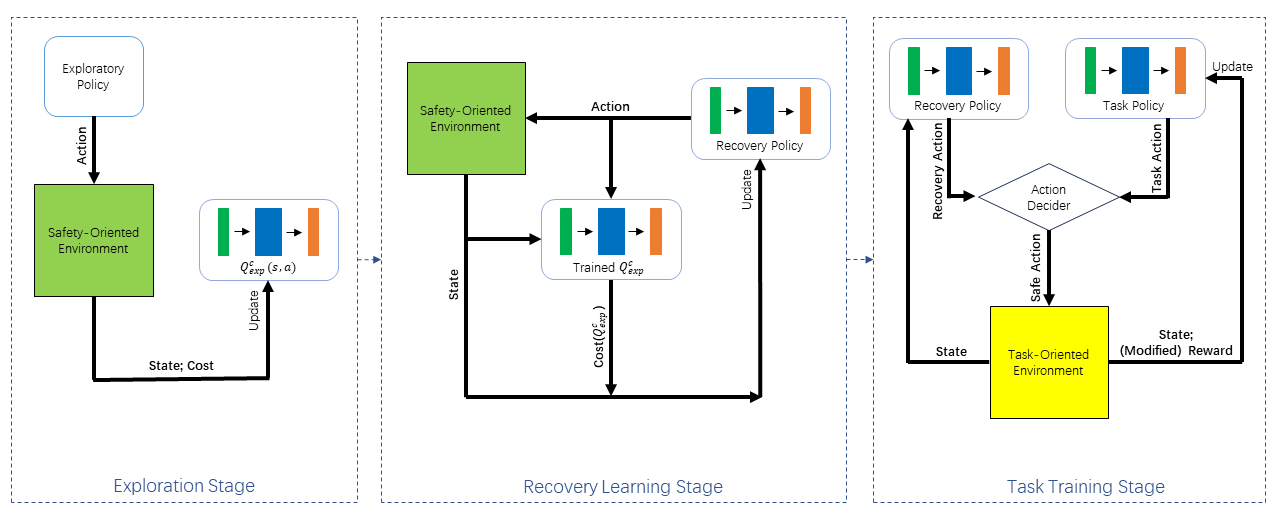}
\caption{Architecture of Safe Reinforcement Learning with Safety Critic and Recovery Policy -- TU-Recovery Architecture.}
\label{fig:Architecture}
\end{figure*}

Recently, there are some researches about using controllers for safe reinforcement learning. The idea is that the exploration process is modified, and the scope of exploration is limited to the safe area, because of the existence of safety controllers; This limitation to the exploration area can prevent the agent from entering dangerous areas.

Shielding RL \cite{DBLP:conf/aaai/AlshiekhBEKNT18} is one of the algorithms use a safety controller. They use prior knowledge, which is modeled by an abstract MDP, to construct a shield, and use this shield to supervise the agent's exploration. In fact, shield is one kind of safety controller. RL-CBF \cite{DBLP:conf/aaai/ChengOMB19} make the assumption that environment dynamics is linear combination of nominal dynamics and unknown dynamics, and the unknown dynamics is modeled as Gaussian process. Optimization problems are constructed using barrier functions. They define a controller to project unsafe actions into safe action space, by solving these optimization problems. Also, safety layer \cite{DBLP:journals/corr/abs-1801-08757} is proposed to play the role of safety controller. They approximate the cost function to the first order of action, then they establish a safety layer to solve a quadratic program problem to project actions into safe action space. SEditor \cite{DBLP:journals/corr/abs-2201-12427} is an extended version of safety layer. In their work, projection from whole action space to safe action space is implement by a learned policy.

There are some algorithms in which the ideas are similar to implement a safety controller, but they do not explicitly define a safety controller. Leave-no-Trace \cite{DBLP:journals/corr/abs-1711-06782} learns a task policy and a reset policy simultaneously. The reset policy is used to provide safety aborts and restore the agent to initial states when task policy is about to enter dangerous states. Recovery RL \cite{DBLP:journals/corr/abs-2010-15920} also learn two policies -- task policy and recovery policy. If an action proposed by task policy is considered as dangerous action, agent will execute the action proposed by recovery policy. DESTA \cite{DBLP:journals/corr/abs-2110-14468} has similar idea to Recovery RL -- it utilizes a task policy and a safety policy. An impulse controller is used to choose between the task action and the safety action.

Our safe RL architecture is greatly inspired by Leave-no-Trace and Recovery RL, but with recovery policy learned before task training. Furthermore, we train the recovery policy in a task-unaware way, which means the learned recovery policy is a general, task-free guiding policy.

There have been some work on utilizing augmented reward for safe reinforcement learning, for example, RCPO \cite{DBLP:conf/iclr/TesslerMM19}. However, the idea of auxiliary reward is far more widely used in other RL fields than safe RL field.  In this work, the idea of using auxiliary reward is mainly inspired by Episodic Curiosity \cite{DBLP:conf/iclr/SavinovRVMPLG19}. However, the purpose of using auxiliary reward is different between Episodic Curiosity and this work: In their work, auxiliary reward is used to deal with sparse reward problem, while auxiliary reward is used to cope with adversarial phenomenon in here.

\section{Safe Reinforcement Learning Architecture}

The safe RL architecture proposed in this paper is shown in figure \ref{fig:Architecture}. We define a three-stage workflow for our architecture, namely the exploration stage, the recovery learning state, and the task training stage. Environments in the first two stages are safety-oriented, while environments in the last stage are task-oriented. In other words, agent receives only safety-related signals (cost signals) and no task-related signals (reward signals) in the first two stages, while it receives task-related signals in the last stage. Although in the last stage the agent need not to receive any cost signal, we design the task-oriented environments to return cost signals, for recording safety violations and evaluating algorithm performances. The recovery policy and safety critic are trained in a \textbf{T}ask-\textbf{U}naware way, so we refer to this architecture as TU-Recovery Architecture.

\subsection{Exploration Stage}

Exploration stage is showed as the left part of figure \ref{fig:Architecture}. The purpose of exploration stage is to learn a safety critic. An exploratory policy $\pi_{epx}$ is used to interact with the environment. During the interaction, we learn the state-action value function of the exploratory policy, according to the following Bellman equation:

\begin{equation}
    \label{equ:CostBellman}
    \begin{split}
    &Q_{exp}^{c}(s_{t},a_{t})=c_{t} \\
    & +(1-c_{t})\gamma \mathop{\mathbb{E}}\limits_{\tau \sim \pi_{exp}} \left[ Q_{exp}^{c}(s_{t+1},a_{t+1})\vert s_{t},a_{t} \right],
    \end{split}
\end{equation}
In practice, we use sampled trajectories to approximate the expectation, and train $Q_{exp}^{c}$ by minimizing the MSE loss of LHS and RHS. This equation is used in the same way in Recovery RL and its previous research \cite{DBLP:journals/corr/abs-2010-15920,DBLP:journals/corr/abs-2010-14603}.

After training, the learned function $Q_{exp}^{c}$ is considered as a safety critic. The larger value of $Q_{exp}^{c}(s,a)$ means the larger probability of the agent to enter dangerous areas after taking action $a$ in state $s$.

Another part needs to be specified in exploratory stage is the exploratory policy. In practice we use random policy as exploratory policy. The reasons of choosing random policy are: First, random policy brings strong exploratory, which helps to train a good safety critic; Second, random policy is easy to implement, and it could save more computing space and time.

\subsection{Recovery Learning Stage}

The middle part of figure \ref{fig:Architecture} illustrates how recovery learning stage works. In recovery learning stage, a recovery policy is trained to minimize the safety critic, which is trained to convergence in the previous stage.

This stage is like a traditional RL procedure, where an agent interact with an safety-oriented environment, except that the per-step reward is not directly given by the environment, but given by the safety critic. Suppose that recovery policy takes an action $a$ in state $s$, then it will receive a reward of $-Q_{exp}^{c}(s,a)$. Note that we use negative of the critic as the per-step reward. This allows us to use traditional RL algorithms to optimize recovery policy, because traditional RL algorithms learns to maximize rewards, which is equivalent to minimize safety critic in this case.

\subsection{Task Training Stage}

As shown in the right part of figure \ref{fig:Architecture}, there are four parts in task training stage -- a task policy, a recovery policy, an action decider, and a task-oriented environment.

The basic idea is to train the task policy under the supervision of the recovery policy and the action decider. Suppose that the agent is in a state $s$, the task policy propose a task action, $a_{task}$, and the recovery policy propose a recovery action, $a_{rec}$. The action decider should choose between $a_{task}$ and $a_{rec}$. Intuitively, the decider should choose task action when it is considered to be safe, otherwise choose recovery action. We follow the implementation scheme of Recovery RL, which use the safety critic $Q_{exp}^{c}$ to make decision: 

\begin{equation}
    \label{equ:ActionDecision}
    a=
    \begin{cases}
    a_{task}, \quad {Q}_{exp}^{c}(s,a_{task}) \leq d \\
    a_{rec}, \quad else
    \end{cases}
\end{equation}
where $d$ is a hyperparameter threshold.

From the view of task agent, recovery policy and action decider can be considered as part of the environment, so it can be considered that recovery policy and action decider change the dynamics of the environment to which task agent interact. This procedure corresponds to figure \ref{fig:Safe_RL}, where the safety controller is composed of the recovery policy and the action decider. Furthermore, according to equation \ref{equ:ActionDecision}, the controller only changes the proposed task action when the risk is beyond some threshold. This could be thought as a kind of \emph{hard intervention}, which replace the task actions with recovery actions in some highly dangerous areas in order to get out of these areas quickly.

\section{Learning Recover Actions through Auxiliary Reward}

\subsection{Motivation}

According to definitions from the last part, task policy learns to maximize reward, while recovery policy helps to restore from high risk areas. It could be observed that sometimes these two policies play against each other, which could cause the agent to get stuck in a small range of states. We refer to this phenomenon as \emph{adversarial phenomenon}.

Figure \ref{fig:Motivation} shows an example of adversarial phenomenon, where a robot (presented as red points) is supposed to reach the target area (presented as green circles), while avoiding collisions with an obstacle (presented as blue circles). Light orange in the figure shows the recovery zone -- agent in this zone will take recovery actions instead of task actions. Task actions and recovery actions are shown as black arrows and green arrows, respectively. Note that the task action and the recovery action are opposite in directions, which makes the agent move back and forth repeatedly at the boundary of the recovery zone.

\begin{figure}[t]
\centering
\includegraphics[width=0.9\columnwidth]{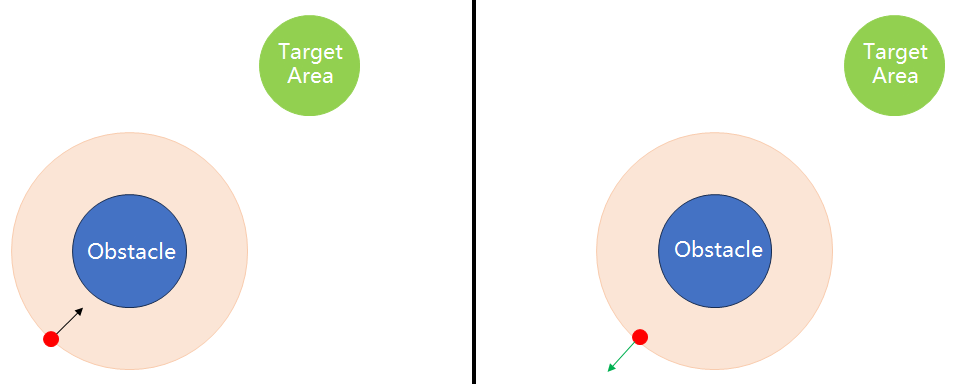}
\caption{Adversarial Phenomenon near the Border of the Recovery Zone}
\label{fig:Motivation}
\end{figure}

\begin{figure}[t]
\centering
\includegraphics[width=0.9\columnwidth]{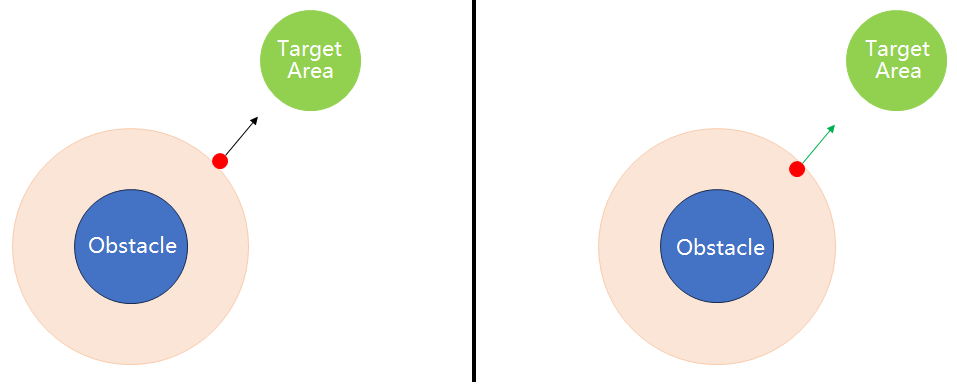}
\caption{Collaborative Phenomenon near the Border of the Recovery Zone}
\label{fig:Motivation2}
\end{figure}

Adversarial phenomenon usually happens in relatively hard tasks. For example, as shown in figure \ref{fig:Motivation}, the agent has to navigate around the obstacle, which is in the middle of the straight line between the agent and the target. However, situations could be totally different in a simple task. As figure \ref{fig:Motivation2} shows, an obstacle is placed in the opposite direction to the target, and there are no obstacles on the way for the agent to reach the target. This is the case where task action and recovery action "agree with" each other. Both task action and recovery action meet the purpose of both task policy and recovery policy -- moving towards the target and moving away from high risk area. We refer to this phenomenon as \emph{collaborative phenomenon}, because task actions and recovery actions help each other in this case.

Adversarial phenomenon can cause great performance degradation during task training stage. It makes the agent stuck around one point without further movements. In experiments we find that this phenomenon of agent stuck can happen even when the policies are trained to convergence. On the other hand, we consider collaborative phenomenon as an advantage, because it accelerates the algorithm convergence. Therefore, the goal is to mitigate the impact of adversarial phenomenon, while enhance the impact of collaborative phenomenon. We propose to utilize an auxiliary reward to meet this goal, as will be described in the following.

\subsection{Auxiliary Reward}

Consider the interactions between the agent and the environment during task training stage. At any time step $t$, the agent receives a reward signal $r_{t}$. An auxiliary reward $b_{t}$ is added to $r_{t}$, resulting in an augmented reward: 

\begin{equation}
    \label{equ:AugmentedReward}
    \hat{r_{t}}=r_{t}+b_{t}.
\end{equation}
Task policy is trained with augmented rewards, rather than original rewards.

The idea behind auxiliary reward is to force the task policy to learn recovery actions in high risk areas. Suppose that the agent is in a high risk state, then the auxiliary reward should give a high value when $a_{task}$ proposed by task policy is close to $a_{rec}$ proposed by recovery policy, otherwise give a small value. However, when the agent is in a low risk state, which means the agent is far from dangerous areas, the auxiliary reward is not supposed to work on the augmented reward, because finishing task is the only thing an agent needs to consider when it is in safe areas.

Based on the above idea, we consider the auxiliary reward of the following form:

\begin{equation}
    \label{equ:AuxiliaryReward}
    b(s,a)=\alpha f(D^{c}(s)) k(a,a_{rec}),
\end{equation}
where $\alpha$ is a positive scaling parameter, $D^{c}$ is a safety critic indicates the risk of a state (the higher risk of a state, the greater value of $D^{c}$), $f(\cdot)$ is a monotonically increasing function on $\mathbb{R}$, ranging from $[0,1]$, and $k(\cdot,\cdot)$ is a measure of how close between two actions.

Ideally, $f$ could be chosen as indicator function: 

\begin{equation}
    \label{equ:IndicatorF}
    f(x)=
    \begin{cases}
        1, \quad x>d \\
        0, \quad x \leq d.
    \end{cases}
\end{equation}
Under some suitable assumptions, a theoretical result about one-step optimization of augmented reward is proposed as follows.

\begin{proposal}
\label{prop:OneStepOptimization}
Suppose that a policy interact with an environment, with augmented reward defined by equation \ref{equ:AugmentedReward}. The auxiliary reward is defined by equation \ref{equ:AuxiliaryReward}, where $f(\cdot)$ is defined by \ref{equ:IndicatorF}, and $k(\cdot,a_{rec})$ is a kernel function maximized at $a=a_{rec}$. Then when $\alpha$ tends to be infinite, for any state $s$ with $D^{c}(s)>d$, the action that maximize one-step augmented reward is given by $a^{*}=a_{rec}$; For any state $s$ with $D^{c}(s) \leq d$, the action that maximize one-step augmented reward is given by $a^{*}=\mathop{argmax}\limits_{a} \ r(s,a)$.
\end{proposal}
\begin{proof}
First, consider a state $s$ with $D^{c}(s)>d$. From equation \ref{equ:IndicatorF}, the augmented reward can be written as $\hat{r}(s,a)=r(s,a)+\alpha k(a,a_{rec})$. Note that maximizing this augmented reward is equivalent to solve a multi-objective maximization problem, with $r(s,a)$ and $k(a,a_{rec})$ as objectives, using additive weighting method. So $\alpha$ can be considered as the weight of $k(a,a_{rec})$ in the problem. When $\alpha$ tends to be infinite, the problem degenerates into maximizing $k(a,a_{rec})$.

Second, consider a state $s$ with $D^{c}(s) \leq d$. It is easy to deduce that maximizing $\hat{r}(s,a)$ is equivalent to maximizing $r(s,a)$ in this case, because the auxiliary reward is always zero, according to equation \ref{equ:IndicatorF}.
\end{proof}
Proposal \ref{prop:OneStepOptimization} gives an insight of how auxiliary reward affect a policy's behavior. If the agent is in a high risk state and $\alpha$ is big enough, action that maximize one-step augmented reward will tend to be close to the recovery action; If the agent is in a low risk state, action that maximize one-step augmented reward will be the same as the one that maximize one-step original reward. Although proposal \ref{prop:OneStepOptimization} is about one-step optimization, which only works with greedy policy, it is reasonable to believe that the augmented reward can also help to enhance the performance of a long-term-concerned policy.

In practice, we use $Q_{exp}^{c}(s,a)$ instead of $D^{c}(s)$ in equation \ref{equ:AuxiliaryReward}. We will show in experiments that a state-action based function is better safety critic than a state based function.

$\left<a,a_{rec}\right>$ is used as $k(a,a_{rec})$ in equation \ref{equ:AuxiliaryReward}. Dot product between two actions indicates the "agreement" of these actions to each other. An action that is in the same direction to $a_{rec}$ shows strong agreement with $a_{rec}$, therefore tends to give a great positive value of auxiliary reward. On the other hand, an action that is in the opposite direction to $a_{rec}$ shows strong disagreement with $a_{rec}$, therefore tends to give a great negative value of auxiliary reward.

\begin{figure}[t]
\centering
\includegraphics[width=0.9\columnwidth]{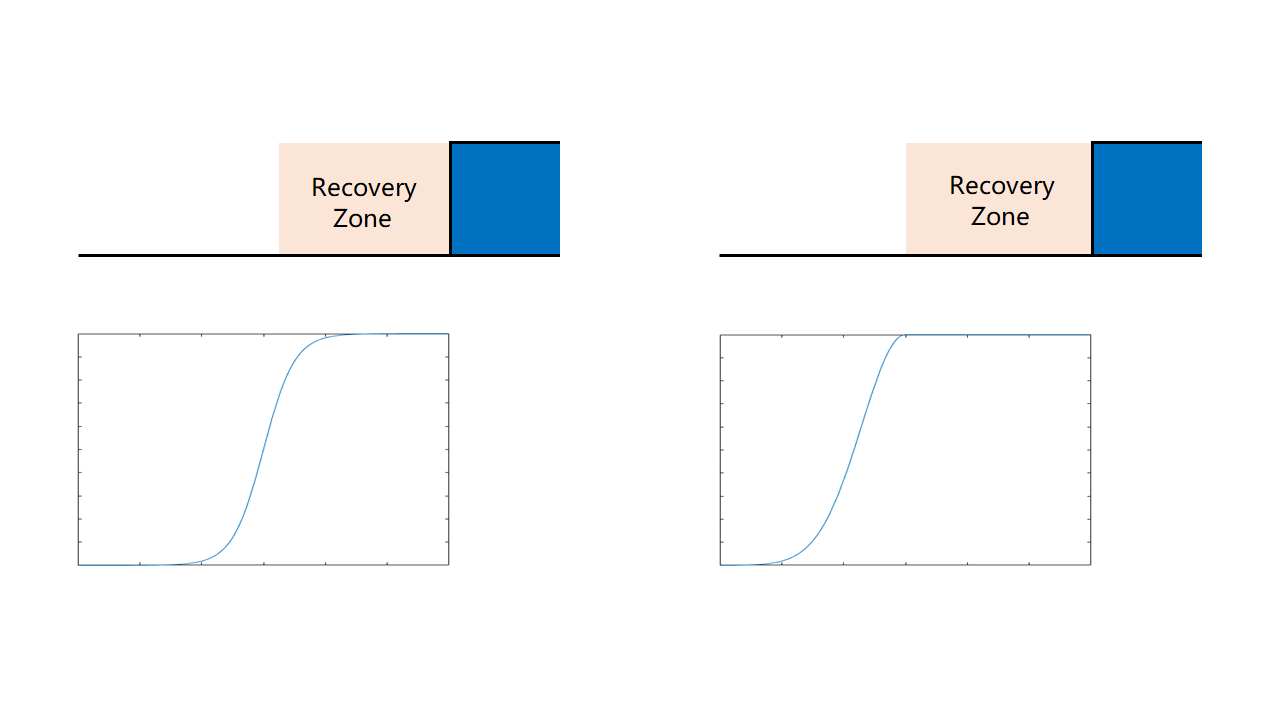}
\caption{Example function graph of $f$. Left: SL function. Right: GC function.}
\label{fig:FDiag}
\end{figure}

Moreover, continuous and smoothed versions of indicator function, rather than the original indicator function is used as $f$. We consider the following two functions: 

\begin{itemize}
    \item Sigmoid function with linear transformation, with hyperparameters $a$ and $b$: 
    \begin{equation}
        f(x)=\frac{1}{1+exp(-(ax+b))}
    \end{equation}
    \item Piece-wise Gaussian and constant function, with $\sigma$ as hyperparameter:
    \begin{equation}
        f(x)=
        \begin{cases}
            1, \quad x>d \\
            exp(-\frac{(x-d)^{2}}{\sigma^{2}})
        \end{cases}
    \end{equation}
\end{itemize}
For simplicity, we refer to the first function as \emph{SL function}, and refer to the second function as \emph{GC function}. Figure \ref{fig:FDiag} shows examples of the two functions. The main reason of using these functions rather than indicator function is to avoid intense value change around some states, and make the learning procedure more stable.

Auxiliary rewards can be thought as a kind of \emph{soft intervention}. It change the task policy's behavior in a gradual manner, which is different from how hard intervention affect the task policy.

\section{Experiments}

\subsection{Environment}

\begin{figure}[t]
\centering
\includegraphics[width=0.9\columnwidth]{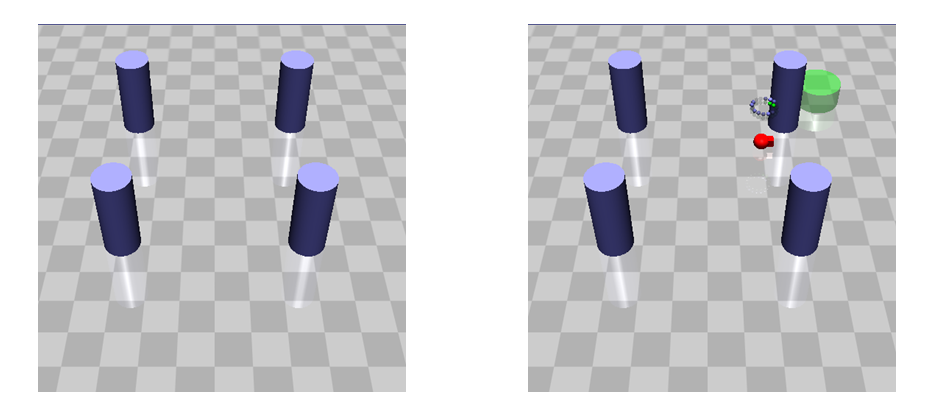}
\caption{Environment Used in Experiments. Left: Obstacles layout. Right: Environment layout at one moment. Obstacles are represented by blue pillars, robot is represented by a red sphere, and target area is represented by green pillars.}
\label{fig:Env}
\end{figure}

The experiment environment is illustrated in figure \ref{fig:Env}. In practice, we construct this environment based on safety-gym \cite{ray2019benchmarking}. It is worth emphasizing that, the start positions and target positions are randomly initialized during training stage.

\subsection{Metrics}

The randomness in start positions and target positions may result in high variance results, so the metrics for algorithms performance should be carefully chosen.

Rewards and costs are the values that most directly show the performance of algorithms. We use cumulative rewards and costs per 1000 steps as basic metrics.

For two algorithms, if one is greater in rewards while less in costs than the other, we say one \emph{dominates} the other. However, sometimes there are no dominance between two algorithms. In these case, we compare two algorithms by the cumulative \textbf{R}eward-to-\textbf{C}ost ratio, abbreviated as \emph{RC-ratio}. During a long period of training procedure, the ratio between the \textbf{M}aximum cumulative \textbf{R}eward and the maximum cumulative \textbf{C}ost during training, abbreviated as \emph{MRC-ratio}, is used. The ratio between the \textbf{A}verage cumulative \textbf{R}eward and the average cumulative \textbf{C}ost during training, abbreviated as \emph{ARC-ratio}, is also used as a metric.

\subsection{Results}

\noindent\textbf{a. Results of Task Training Stage.}
We test four algorithms -- an unconstrained method, TU-Recovery method and TU-Recovery method with auxiliary reward. The learning curves for cumulative rewards and costs are shown in figure \ref{fig:AlgorithmResult}. We also report the MRC-ratio and ARC-ratio, as shown in table \ref{tab:AlgorithmResult}. All policies are trained using SAC algorithm \cite{DBLP:conf/icml/HaarnojaZAL18,DBLP:journals/corr/abs-1812-05905}, and all results are the averages of 5 runs.

\begin{figure}[ht]
\centering
\includegraphics[width=0.9\columnwidth]{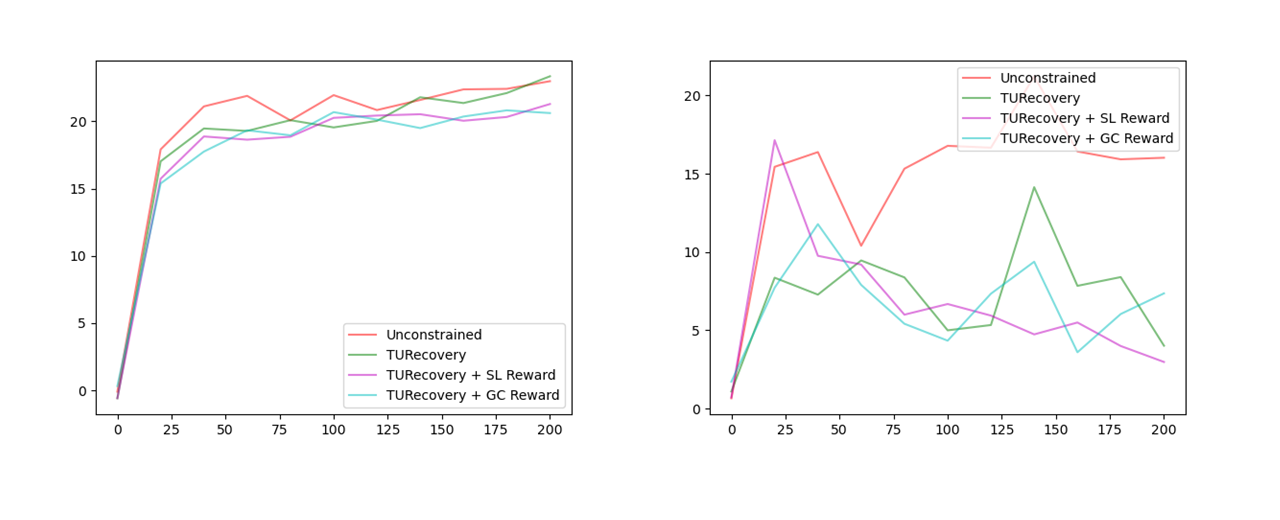}
\caption{Cumulative rewards and costs for unconstrained method and our methods. Left: Learning curves of cumulative rewards. Right: Learning curves of cumulative costs.}
\label{fig:AlgorithmResult}
\end{figure}

\begin{table}[ht]
\caption{Task Training Results for Unconstrained Method and Our Methods.}
\centering
\begin{tabular}{lll}
    \hline
    Algorithm & MRC-ratio & ARC-ratio \\
    \hline
    Unconstrained & 0.800 & 1.321 \\
    TU-Recovery & 1.232 & 2.564 \\
    TU-Recovery + SL Reward & 1.054 & \textbf{2.675} \\
    TU-Recovery + GC Reward & \textbf{1.421} & 2.669 \\
    \hline
\end{tabular}
\label{tab:AlgorithmResult}
\end{table}

It can be easily seen from figure \ref{fig:AlgorithmResult} that, TU-Recovery methods (including the ones with auxiliary rewards) significantly reduce constraint violations comparing to unconstrained method. Results from table \ref{tab:AlgorithmResult} also indicate that our methods outperform unconstrained method. Furthermore, table \ref{tab:AlgorithmResult} shows that auxiliary rewards can improve the performance, especially for the GC reward (4th row in table \ref{tab:AlgorithmResult}), with which the algorithm outperforms the one without auxiliary reward (2nd row in \ref{tab:AlgorithmResult}) in both MRC-ratio and ARC-ratio.

\noindent\textbf{b. How Do Auxiliary Reward Improve Task Policy?}
To give an insight of how auxiliary rewards affect the behaviors of task policy, we test the trained task policies, and measure their performance by cumulative reward, cumulative cost, and RC-ratio. In this experiment, interventions of recovery policy is disabled, so that all results reflect the pure performance of trained task policies. Results are shown in table \ref{tab:TaskPolicy}. For each task policy, 10 different seeds is used to conduct 10 runs, so these results are the averages of 10.

\begin{table}[ht]
\caption{Performance of The Trained Task Policies.}
\centering
\begin{tabular}{llll}
    \hline
    Algorithm & Reward & Cost & RC-ratio \\
    \hline
    Unconstrained & 22.136 & 15.84 & 1.397 \\
    TU-Recovery & \textbf{23.682} & 8.0 & 2.960 \\
    TU-Recovery + SL Reward & 22.212 & \textbf{4.32} & \textbf{5.142} \\
    TU-Recovery + GC Reward & 22.590 & 5.8 & 3.895 \\
    \hline
\end{tabular}
\label{tab:TaskPolicy}
\end{table}

It is clear that task policies trained with auxiliary reward (3rd, 4th rows in \ref{tab:TaskPolicy}) is significantly less in constraint violations, while slightly less in reward gaining, compared to task policy trained with original reward (2nd row in \ref{tab:TaskPolicy}), resulting in higher RC-ratio (as shown in the 4th column).

\subsection{Ablation Experiments}

\noindent\textbf{a. The Reason of Using Exploratory Policy.}
To give a justification for using Q function of exploratory policy as safety critic, we conduct an experiment to compare different safety critics. We consider two other safety critics in addition to $Q_{exp}^{c}$, as will described in the following.

A simple plan to build a safety critic is to directly train the recovery policy to minimize expected discounted cost, then use Q function of this recovery policy, denoted by $Q_{d\_rec}^{c}$, as the safety critic. One benefit of this plan is that, exploratory stage is combined with recovery learning stage, so the three-stage training workflow is simplified to a two-stage procedure.

We also consider using the distance from the robot to an obstacle as safety critic. There are more than one obstacles in the environment, so the minimum distance from the robot to these obstacles, denoted by $D_{min}^{c}$, is used. Note that this is a handcrafted safety critic, and normally the robot has no knowledge about this value, so this critic only acts as an ideal state-based safety critic. When using this safety critic, only recovery learning stage and task training stage need to be executed.

\begin{figure}[ht]
\centering
\includegraphics[width=0.9\columnwidth]{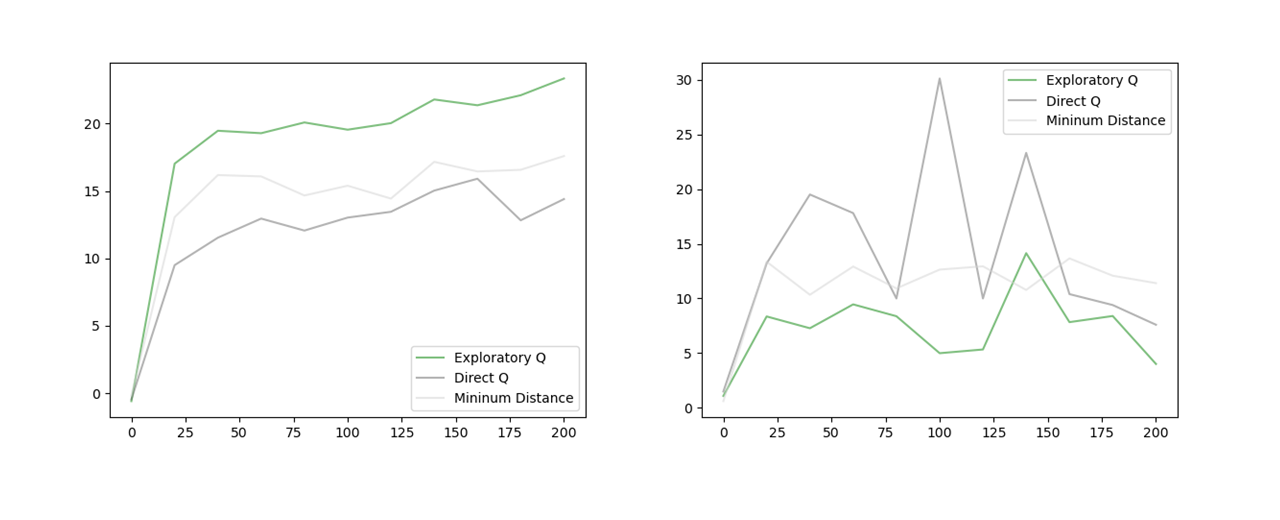}
\caption{Learning Curves of TU-Recovery with Different Safety Critics. Left: Cumulative reward curves. Right: Cumulative cost curves.}
\label{fig:SafetyCritic}
\end{figure}

\begin{figure}[ht]
\centering
\includegraphics[width=0.9\columnwidth]{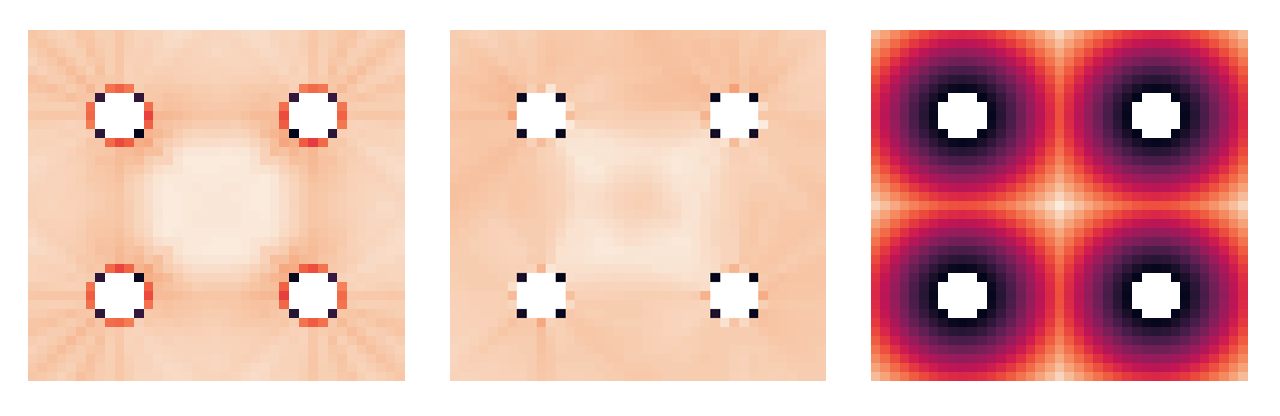}
\caption{Heatmaps of Different Safety Critics. Left: Q function of exploratory policy. Middle: Q function of directly trained recovery policy. Right: Minimum distance from current position to an obstacle.}
\label{fig:HeatMaps}
\end{figure}

We compare the task learning curve of TU-Recovery with three different safety critics, as shown in figure \ref{fig:SafetyCritic}. The conclusion is that the one using exploratory policy dominates the one using minimum distance, and the one using minimum distance dominates the one using directly recovery policy.

We give an insight about the above conclusion by providing the heatmaps of three safety critics, as figure \ref{fig:HeatMaps} shows. All heatmaps about Q functions are drawn by masking the actions to zeros. By comparison between the heatmaps of $Q_{exp}^{c}$ (left part in figure \ref{fig:HeatMaps}) and $Q_{d\_rec}^{c}$ (middle part in figure \ref{fig:HeatMaps}), it makes sense that $Q_{exp}^{c}$ is a better safety critic than $Q_{d\_rec}^{c}$: $Q_{exp}^{c}$ shows homogeneously decreasing in all directions as the distance from the obstacle increases, while $Q_{d\_rec}^{c}$ does not. According to this point of view, it seems that $D_{min}^{c}$ should be better safety critic than $Q_{exp}^{c}$. Why the results from figure \ref{fig:SafetyCritic} contradict this view? We argue that this is because $Q_{exp}^{c}$ is state-action based critic, while $D_{min}^{c}$ is state based critic. A state-action based critic takes the actions into account, and is tend to give more reasonable values than a state based critic. Therefore, the results also justify using Q function as safety critic.

\noindent\textbf{b. Necessity of Using Hard Intervention.}
We conduct a experiment to see whether hard intervention is necessary. TU-Recovery with both hard intervention and soft intervention (auxiliary reward) is compared to its soft-intervention-only counterpart. The result is shown in figure \ref{fig:HIResult}. It can be seen that soft intervention help the task policy to learn to be safe gradually, resulting in better performance than unconstrained method, but a hard intervention is still necessary to ensure safety during the whole task training.

\begin{figure}[ht]
\centering
\includegraphics[width=0.9\columnwidth]{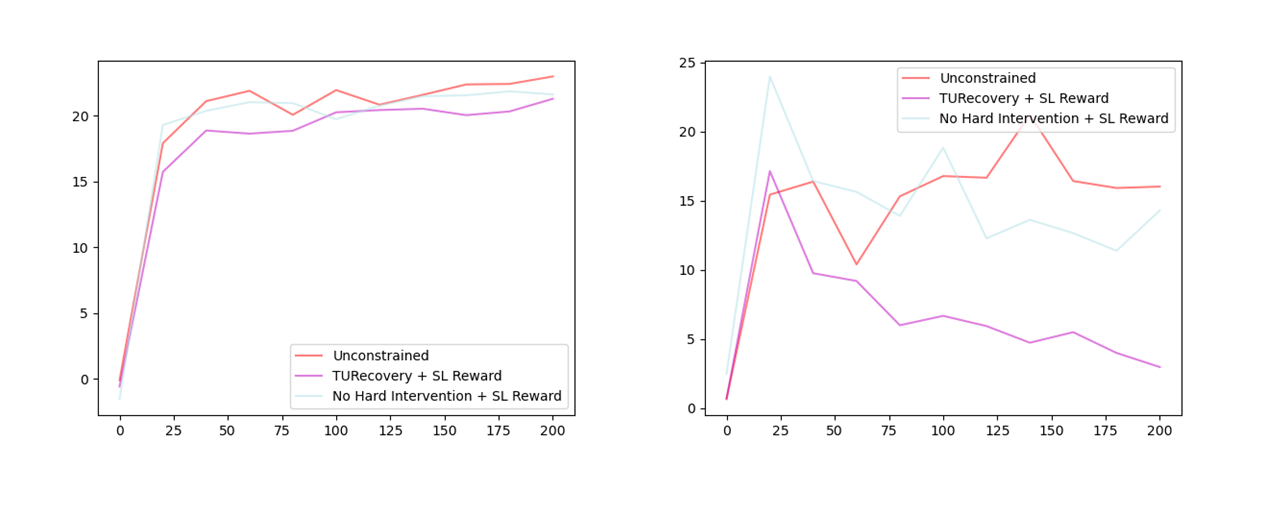}
\caption{Ablations of Hard Intervention. Left: Cumulative reward curves. Right: Cumulative cost curves.}
\label{fig:HIResult}
\end{figure}

\section{Conclusion}

We propose a three-stage framework for safe reinforcement learning, named TU-Recovery Architecture. The framework can construct safety constraints by learning, avoiding handcrafting safety constraints. It is demonstrated that our framework outperformed unconstrained counterpart in task training. Adversarial phenomenon may degrade the performance during task training, so auxiliary rewards is proposed to mitigate this issue. Experiments show that auxiliary rewards can efficiently help the task policy to learn recovery actions.


\section{Acknowledgements}

We express our acknowledgement to Zhejiang University and all people provided us with technical supports and helpful suggestions.

\bigskip

\appendix

\nobibliography*

\nocite{sutton2018reinforcement}
\nocite{DBLP:conf/nips/SuttonMSM99}
\nocite{DBLP:journals/corr/BrockmanCPSSTZ16}

\nocite{DBLP:journals/corr/AmodeiOSCSM16}
\nocite{DBLP:journals/jmlr/GarciaF15}
\nocite{DBLP:journals/corr/AchiamHTA17}
\nocite{DBLP:conf/iclr/YangRNR20}
\nocite{DBLP:conf/nips/YuYKW19}
\nocite{DBLP:conf/icml/WachiS20}
\nocite{DBLP:journals/corr/SaundersSSE17}

\nocite{DBLP:conf/taros/JocasZKGS22}
\nocite{DBLP:conf/nips/0010VR20}
\nocite{DBLP:journals/jmlr/ChowGJP17}
\nocite{DBLP:journals/neco/ShenTSO14}
\nocite{DBLP:conf/icra/ChenLZXD0Z21}
\nocite{DBLP:journals/corr/MnihKSGAWR13}
\nocite{DBLP:journals/corr/LillicrapHPHETS15}
\nocite{DBLP:journals/corr/SchulmanWDRK17}
\nocite{DBLP:conf/icml/SchulmanLAJM15}
\nocite{DBLP:journals/corr/abs-2205-10330}

\bibliography{aaai24}

\end{document}